\documentclass[12pt]{article}
\usepackage[utf8]{inputenc}

\def \AuthorA {Raymond Feng}
\def \EmailA {rfeng2004@gmail.com}
\def \AuthorG {Jesse Geneson}
\def \EmailG {geneson@gmail.com}
\def \AuthorB {Andrew Lee}
\def \EmailB {leeandrew1029@gmail.com}
\def \AuthorC {Espen Slettnes}
\def \EmailC {espen@slett.net}

\usepackage{amsfonts, amsthm, amsmath, amssymb}
\usepackage{cancel, textcomp}
\usepackage[mathscr]{euscript}
\usepackage{float}
\usepackage[hidelinks]{hyperref}
\usepackage[none]{hyphenat}
\usepackage{mathtools}
\usepackage{microtype}
\usepackage{physics}
\usepackage[nointegrals]{wasysym}
\usepackage{etoolbox}
\apptocmd{\lim}{\limits}{}{}

\usepackage{geometry}
\usepackage[backend=biber,bibstyle=numeric,maxbibnames=99]{biblatex}
\bibliography{citations}

\newcommand{\email}[1]{\thanks{\href{mailto:#1}{\texttt{#1}}}}

\newcommand{\floor}[1]{\left\lfloor#1\right\rfloor}

\newcounter{idx}
\newtheorem{thm}[idx]{Theorem}

\newtheorem{lem}[idx]{Lemma}
\newtheorem{cor}[idx]{Corollary}
\theoremstyle{definition}
\newtheorem{dfn}[idx]{Definition}

\newcommand{\opt}{\operatorname{opt}}
\newcommand{\cart}{\operatorname{CART}}

\newtheorem{remark}[idx]{Remark}

\newcommand{\EE}{\mathbb E}
\newcommand{\PP}{\mathbb P}

\renewbibmacro{in:}{%
  \ifentrytype{article}
    {}
    {\bibstring{in}%
     \printunit{\intitlepunct}}}

\DeclareFieldFormat{pages}{#1}

\begin{document}

\title{Sharp bounds on the price of bandit feedback for several models of mistake-bounded online learning}
\author{\AuthorA \email{\EmailA} \and \AuthorG \email{\EmailG} \and \AuthorB \email{\EmailB} \and \AuthorC \email{\EmailC}}

\maketitle
        
        \begin{abstract}
            
            We determine sharp bounds on the price of bandit feedback for several variants of the mistake-bound model. The first part of the paper presents bounds on the $r$-input weak reinforcement model and the $r$-input delayed, ambiguous reinforcement model. In both models, the adversary gives $r$ inputs in each round and only indicates a correct answer if all $r$ guesses are correct. The only difference between the two models is that in the delayed, ambiguous model, the learner must answer each input before receiving the next input of the round, while the learner receives all $r$ inputs at once in the weak reinforcement model. 
            
            In the second part of the paper, we introduce models for online learning with permutation patterns, in which a learner attempts to learn a permutation from a set of permutations by guessing statistics related to sub-permutations. For these permutation models, we prove sharp bounds on the price of bandit feedback.
            
        \end{abstract}






\section{Introduction}
We investigate several variants of the mistake-bound model \cite{Littlestone1988}. In the standard model \cite{AuerLong1999,LONG2020159} (also called the standard strong reinforcement learning model \cite{AuerLong1999}), a learner attempts to classify inputs (in the set $X$) with labels (in the set $Y$) based on a set of possible functions $f:X\to Y$ in $F$. The learning proceeds in rounds, and in each round the adversary gives the learner an input, and the learner must then guess the corresponding label. After each round, the adversary informs the learner of the correct answer (and therefore whether the learner was right or wrong). A variant of this model is the standard weak reinforcement learning model \cite{AuerLong1999,Auer1995}, where the adversary only tells the learner \textit{yes} if they were correct and \textit{no} otherwise. This variant is also commonly called the \emph{bandit model} \cite{Crammer2012MulticlassCW,DaniHayesKakade,DH2013,Hazan2011NewtronAE,LONG2020159}.

For any learning scenario, we generally let $\opt_{\operatorname{scenario}}(F)$ represent the optimal worst case number of mistakes that a learning algorithm can achieve \cite{AuerLong1999}. For example, for the bandit model and the standard model, the optimal worst case performances of learning algorithms would be denoted $\opt_{\operatorname{weak}}(F)=\opt_{\operatorname{bandit}}(F)$ and $\opt_{\operatorname{std}}(F)=\opt_{\operatorname{strong}}(F)$ respectively. There are some obvious inequalities that follow by definition, such as $\opt_{\operatorname{strong}}(F)\leq\opt_{\operatorname{weak}}(F)$, just from the fact that the learner has strictly more information in one scenario compared to the other.

In \cite{AuerLong1999}, Auer and Long defined the $r$-input delayed, ambiguous reinforcement model and compare it to a modified version of the standard weak reinforcement model (henceforth called the $r$-input weak reinforcement model). The delayed, ambiguous reinforcement model is a situation where the learner receives a fixed number ($r$) of inputs each round, and each input is given to the learner after they have answered the previous one. On the other hand, the learner receives all $r$ inputs at once for each round in the modified weak reinforcement model. In both situations, at the end of every round of $r$ inputs, the adversary says \textit{yes} if the learner answered all $r$ inputs correctly and \textit{no} otherwise. To compare the two situations, Auer and Long defined $\cart_r(F)$ (where $F$ is a set of functions $f:X\to Y$) to be a set of functions $f':X^r\to Y^r$ where each $f\in F$ has a corresponding $f'\in\cart_r(F)$ such that for any $x_1,x_2,\dots,x_r\in X$, we have $f'((x_1,x_2,\dots,x_r))=(f(x_1),f(x_2),\dots,f(x_r))$. They used $\opt_{\operatorname{weak}}(\cart_r(F))$ to represent the optimal worst-case performance in the modified weak reinforcement setting and $\opt_{\operatorname{amb},r}(F)$ to represent the optimal worst-case performance in the $r$-input delayed, ambiguous reinforcement model. Auer and Long proved that the two situations are not equivalent for the learner in \cite{AuerLong1999} by showing that that there is some input set $X$ and set $F$ of functions from $X$ to $\{0,1\}$ such that $\opt_{\operatorname{weak}}(\cart_2(F)) < \opt_{\operatorname{amb},2}(F)$. 

In Sections~\ref{section:boundcartr} and \ref{section:boundambr}, we obtain sharp bounds on the maximum possible multiplicative gap between $\opt_{\operatorname{amb},r}(F)$ and $\opt_{\operatorname{weak}}(\cart_r(F))$. In particular, we show that the multiplicative gap between $\opt_{\operatorname{amb},r}(F)$ and $\opt_{\operatorname{weak}}(\cart_r(F))$ can grow exponentially with $r$, generalizing Auer and Long's result from \cite{AuerLong1999} for $r = 2$. Combined with a bound from \cite{AuerLong1999}, our new result shows that the maximum possible multiplicative gap between $\opt_{\operatorname{amb},r}(F)$ and $\opt_{\operatorname{weak}}(\cart_r(F))$ is $2^{r (1 \pm o(1))}$. Moreover, we give sharp bounds on $\opt_{\operatorname{weak}}(\cart_r(F))$ (in Section~\ref{section:boundcartr}) and $\opt_{\operatorname{amb},r}(F)$  (in Section~\ref{section:boundambr}) for all sets $F$ which are a subset of the non-decreasing functions from $X$ to $\{0,1\}$. 

In a different paper, Long \cite{LONG2020159} also determined sharp bounds comparing the standard model and the standard weak reinforcement model for multi-class functions. Long proved the upper bound $\opt_{\operatorname{bandit}}(F)\leq(1+o(1))(|Y|\ln |Y|)\opt_{\operatorname{std}}(F)$ and constructed infinitely many $F$ for which $\opt_{\operatorname{bandit}}(F)\geq(1-o(1))(|Y|\ln |Y|)\opt_{\operatorname{std}}(F)$ as a lower bound. Geneson corrected an error in the proof of the lower bound \cite{Geneson21}.

In Section~\ref{section:factorgap}, we generalize this result to determine a lower bound and upper bound on the maximum factor gap between $\opt_{\operatorname{weak}}(\cart_r(F))$ and $\opt_{\operatorname{std}}(F)$ for multi-class functions using probabilistic methods and linear algebra techniques. The proof uses techniques previously used for experimental design \cite{rao_hypercubes,rao_factorial} and hashing, derandomization, and cryptography \cite{carter1977,luby_pairwise}. We also determine a lower bound and upper bound on the maximum factor gap between $\opt_{\operatorname{amb,r}}(F)$ and $\opt_{\operatorname{std}}(F)$. The bounds in this section are sharp up to a factor of $r(1+o(1))$.

In Section~\ref{section:permutation}, we define several new models where the learner is trying to guess a function from a set of permutations of length $n$. In the order model, the adversary chooses $r$ inputs, and the learner attempts to guess the corresponding sub-permutation. In the comparison model, the adversary instead chooses $r$ pairs of inputs, and the learner attempts to guess for each pair whether or not it is an inversion. In the selection model, the adversary chooses $r$ inputs, and the learner attempts to guess which has the maximum evaluation. In the relative position model, the adversary chooses $r$ inputs and an additional distinguished input $x$, and the learner attempts to guess the relative position of $x$ in the sub-permutation corresponding to all $r+1$ of these inputs. Finally, in the delayed relative position model, the adversary instead gives the $r$ elements to be compared to $x$ one at a time. We first establish general upper bounds, and then we discuss adversary strategies for a few special families of permutations that resemble sorting algorithms.

Finally, in Section~\ref{section:future}, we discuss some future directions based on the results in this paper.

\section{Bounds on $\opt_{\operatorname{weak}}(\cart_r(F))$}\label{section:boundcartr}
In this section, we establish upper and lower bounds on $\opt_{\operatorname{weak}}(\cart_r(F))$ for families $F$ of non-decreasing functions that are within a constant factor of each other. We show for such families $F$ that $\opt_{\operatorname{weak}}(\cart_r(F))=(1\pm o(1))r\ln(|F|)$.

\begin{dfn}
    Without loss of generality impose an ordering on the set $X$, and call its elements $\{1,2,\dots,|X|\}$. Let $F=\{f_1,f_2.\dots,f_{|F|}\}$ be a subset of the functions from $X$ to $\{0,1\}$. We say that $F$ is \emph{non-decreasing} if every function in $F$ is non-decreasing.
    
    In other words, there are integers $1\leq a_1<a_2<\dots<a_{|F|}\leq |X|+1$ which are the minimum numbers such that $f_i(a_i)=1$ (with the convention that $a_{|F|}=|X|+1$ if $f_{|F|}$ is identically 0) and satisfy the following property for each $1\leq i\leq |F|$:
    \begin{itemize}
        \item If $x>a_i$, then $f_i(x)=1$.
        \item If $x<a_i$, then $f_i(x)=0$.
    \end{itemize}
\end{dfn}

\begin{remark}\label{remark:r+1}
    Let $F$ be non-decreasing. For a function $f\in F$ and any choice of $r$ inputs from the set $X$, there are at most $r+1$ possible values of the corresponding outputs $(f(x_1),\dots,f(x_r))$, namely $(0,0,\dots,0,0)$, $(0,0,\dots,0,1)$, $\ldots$, $(0,1,\dots,1,1)$, and $(1,1,\dots,1,1)$.
\end{remark}


\subsection{Bounds on $\opt_{\operatorname{weak}}(\cart_r(F))$} 

The following theorem establishes an upper bound on $\opt_{\operatorname{weak}}(\cart_r(F))$ by exhibiting a possible learner strategy.

\begin{thm}\label{theorem:cartrupper}
    For non-decreasing $F$, $\opt_{\operatorname{weak}}(\cart_r(F))<(r+1)\cdot\ln(|F|)$.
\end{thm}

\begin{proof}
The learner's strategy: pick the answer that corresponds to the most functions. By Remark \ref{remark:r+1}, we know that there are at most $r+1$ such answers.

Each time the adversary says \textit{no}, if the learner previously knew that there were $T$ possible functions left, then the learner is able to reduce the number of possible functions left by at least $\frac T{r+1}$. Thus, each answer of \textit{no} means that the number of remaining possibilities is multiplied by a factor of at most $\frac r{r+1}$ on each turn.

Then, the learner will make at most \[\log_{\frac{r+1}r}(|F|)=\frac{\ln(|F|)}{\ln\left(1+\frac1r\right)}<\frac{\ln(|F|)}{\frac{1}{r+1}}=(r+1)\cdot\ln(|F|)\] mistakes, as desired.
\end{proof}

Next we establish a lower bound on $\opt_{\operatorname{weak}}(\cart_r(F))$ by exhibiting a strategy for the adversary.

\begin{thm}\label{theorem:cartrlower}
    For non-decreasing $F$, $\opt_{\operatorname{weak}}(\cart_r(F))\geq(1-o(1))r\ln(|F|)$ as $|F| \rightarrow \infty$.
\end{thm}

\begin{proof}
    In each round, the adversary will say \textit{no}. When they do this, some functions fail to remain consistent with the answers given by the adversary. Let $F'\subseteq F$ be the current set of functions that are consistent with the answers that the adversary has given so far. Furthermore, with $F'=\{g_1,g_2,\dots,g_{|F'|}\}$, define $1\leq b_1<b_2<\dots<b_{|F'|}\leq |X|+1$ as the minimum numbers such that $g_i(b_i)=1$ (with the convention that $b_{|F'|}=|X|+1$ if $f_{|F'|}$ is identically 0).
    
    In each round, the adversary will choose the inputs \[x_i=b_{i\cdot\left\lceil\frac{|F'|}{r+1}\right\rceil+1}\] for $1\leq i\leq r$. Then, no matter what the learner says, the adversary says \textit{no}. This guarantees that the number of remaining consistent functions decreases by at most $\left\lceil\frac{|F'|}{r+1}\right\rceil$ functions per round. Thus, the adversary can continue for at least \[(1-o(1))\log_{\frac{r+1}r}(|F|) = (1-o(1))\cdot\frac{\ln(|F|)}{\ln(1+\frac1r)}\geq(1-o(1))r\ln(|F|)\] turns. Therefore, the adversary guarantees that they can say \textit{no} at least $(1-o(1))r\ln(|F|)$ times, as desired.
\end{proof}

Combining the above bounds, this implies that for non-decreasing $F$, we have the following theorem:
\begin{thm}\label{theorem:cartr}
    For non-decreasing $F$, $\lim_{r \to \infty} 
    \lim_{|F| \to \infty} \frac{\opt_{\operatorname{weak}}(\cart_r(F))}{r\ln(|F|)} = 1$.
\end{thm}

\section{Bounds on $\opt_{\operatorname{amb},r}(F)$}\label{section:boundambr}
In this section, we show that $\opt_{\operatorname{amb},r}(F)=(1-o(1))2^r\ln(|F|)$ as $|F|, r \rightarrow \infty$ for non-decreasing families $F$. The upper bound in this section is applicable to all families of functions $F$, but the lower bound is only applicable for non-decreasing sets of functions $F$, as with the bounds established in Section~\ref{section:boundcartr}.

We prove the following theorem, a general upper bound for $\opt_{\operatorname{amb},r}(F)$, using a learner strategy that does not make any assumptions about the set $F$. In particular, $F$ does not have to be non-decreasing.

\begin{thm}\label{theorem:ambrupper}
    For all $F$, $\opt_{\operatorname{amb},r}(F)<\min(2^r\ln(|F|),|F|)$.
\end{thm}

\begin{proof}
    For each input that the adversary gives, the learner picks the answer that corresponds to the most functions which were consistent with all answers before the round started and also consistent with all earlier guesses in the current round.
    
    Each time the adversary says \textit{no}, if the learner knew that there were $T$ possible functions remaining before the round started, then they can guarantee that their answer for all $r$ inputs is consistent with at least $\frac T{2^r}$ of those functions. To see this, note that by induction, at least $\frac T{2^k}$ of the functions will be consistent with the $k$ answers they have given so far for each $1\leq k\leq r$. Therefore, each time the adversary says \textit{no} the number of remaining possible functions is multiplied by at most $\frac{2^r-1}{2^r}$. So, the learner makes at most \[\log_{\frac{2^r}{2^r-1}}(|F|)=\frac{\ln(|F|)}{\ln\left(1+\frac1{2^r-1}\right)}<\frac{\ln(|F|)}{\frac{1}{2^r}}=2^r\cdot\ln(|F|)\] mistakes with this strategy.
    
    The learner's strategy to get $\opt_{\operatorname{amb},r}(F)\leq|F|-1$: Each time the adversary says \textit{no}, the learner can eliminate at least 1 function. Once the learner has eliminated $|F|-1$ functions, no more errors will be made.
\end{proof}

For the following lower bound on $\opt_{\operatorname{amb},r}(F)$, we again assume that $F$ is non-decreasing.

\begin{thm}\label{theorem:ambrlower}
    For non-decreasing $F$, $\opt_{\operatorname{amb},r}(F)\geq (1-o(1))(2^r-1)\ln(|F|)$ as $|F| \rightarrow \infty$.
\end{thm}

\begin{proof}
    The adversary will say \textit{no} at the end of each round. For each round, the adversary will choose a series of input values $x_i$ based on the answers given by the learner. In each subround, the next input $x_i$ is determined as follows: suppose that $S$ is the set of all functions that are consistent with all previous adversary answers from past rounds as well as all the answers of the learner from the current round.
    
    Since $S\subseteq F$, we can then set $S=\{g_1,g_2,\dots,g_{|S|}\}$ and define $1\leq b_1<b_2<\dots<b_{|S|}\leq |X|+1$ as the minimum numbers such that $g_i(b_i)=1$ (with the convention that $b_{|S|}=|X|+1$ if $f_{|S|}$ is identically 0).
    
    The adversary then chooses $x_i=b_{\left\lceil\frac{|S|}2\right\rceil}$ for the current subround. This guarantees that at each subround, the number of functions consistent with all of the adversary's previous answers as well as all of the learner's answers in the current round reduces by at most $\left\lceil\frac{|S|}2\right\rceil$. Thus, if $T$ functions were consistent with all of the adversary's previous answers at the beginning of the current round, then at the end of the round, at most $\left\lceil\frac T{2^r}\right\rceil$ functions become inconsistent with the adversary's answers. Here, we are repeatedly using the fact that $\left\lceil\frac{\left\lceil x\right\rceil}n\right\rceil=\left\lceil\frac xn\right\rceil$ for all positive reals $x$ and positive integers $n$.
    
    This means that the adversary can continue to say \textit{no} for at least
    \begin{align*}
        (1-o(1))\log_{\frac{2^r}{2^r-1}}(|F|)&=(1-o(1))\frac{\ln(|F|)}{\ln\left(1+\frac1{2^r-1}\right)} \geq (1-o(1))(2^r-1)\ln(|F|)
    \end{align*}
    turns, as desired.
\end{proof}

Combining the above bounds, this implies that for non-decreasing $F$, we have the following theorem:
\begin{thm}\label{theorem:ambr}
    For non-decreasing $F$, $\lim_{r \to \infty} 
    \lim_{|F| \to \infty} \frac{\opt_{\operatorname{amb},r}(F)}{2^r\ln(|F|)} = 1$.
\end{thm}

Theorem~\ref{theorem:cartr} and Theorem~\ref{theorem:ambr} imply for non-decreasing families of functions $F$ that for sufficiently large $r$ and $|F|$ that learners who are given all inputs at the beginning of each round do exponentially better in $r$ than their counterparts who receive inputs one at a time in each round.

In \cite{AuerLong1999}, Auer and Long proved that $\opt_{\operatorname{amb},r}(F) \le 2 (\ln{2r}) \cdot 2^r \cdot \opt_{\operatorname{std}}(F)$. Since $\opt_{\operatorname{std}}(F) \le \opt_{\operatorname{weak}}(\cart_r(F))$ for all families of functions $F$, this implies the following result when combined with our Theorem~\ref{theorem:cartr} and Theorem~\ref{theorem:ambr}.

\begin{thm}\label{maxgap}
The maximum possible value of $\frac{\opt_{\operatorname{amb},r}(F)}{\opt_{\operatorname{weak}}(\cart_r(F))}$ over all families of functions $F$ with $|F| > 1$ is $2^{r(1 \pm o(1))}$.
\end{thm}

\section{Comparing $\opt_{\operatorname{weak}}(\cart_r(F))$ and $\opt_{\operatorname{amb},r}(F)$ to $\opt_{\operatorname{std}}(F)$ for multi-class functions}\label{section:factorgap}

In \cite{LONG2020159}, Long compared the standard and bandit model for families of multi-class functions, and determined a bound on the maximum multiplicative gap between them. There was an error in Long's proof of the lower bound, but Geneson fixed this error in \cite{Geneson21}. In this section, we bound the maximum multiplicative gap between $\opt_{\operatorname{weak}}(\cart_r(F))$ and $\opt_{\operatorname{std}}(F)$ using similar methods, as well as the maximum multiplicative gap between $\opt_{\operatorname{amb},r}(F)$ and $\opt_{\operatorname{std}}(F)$. For each of the gaps, the proof of the lower bound employs techniques previously used for experimental design, hashing, derandomization, and cryptography \cite{luby_pairwise,carter1977, rao_hypercubes,rao_factorial}. We also adapt the proof of the upper bound in \cite{LONG2020159} to show that our lower bounds are sharp up to a factor of $r(1+o(1))$.

In order to obtain the lower bounds, we prove the following three lemmas which generalize results from \cite{LONG2020159} and \cite{Geneson21}. For the rest of this section, we assume that $p$ is a prime number.

\begin{lem} \label{long1g}
Fix $n \ge 2$, suppose that $z_1, \dots, z_r \in \left\{0, \dots, p-1\right\}$, and let $u_1, \dots, u_r$ each be chosen uniformly at random from $\left\{0, \dots, p-1\right\}^n$. For any $s \in \left\{1, \dots, p-1\right\}^n$, we have \[\Pr(s \cdot u_i = z_i \mod p \text{ for all } 1 \le i \le r) = \frac{1}{p^r}\].
\end{lem}

\begin{proof}
We have 
\begin{align*}
    \Pr(s \cdot u_i = z_i \mod p \text{ for all } 1 \le i \le r) = \\
    \Pr(s_j u_{i,j} = z_i-\sum_{k \neq j} s_k u_{i,k} \mod p \text{ for all } 1 \le i \le r) = \\
    \Pr(u_{i,j} = (z_i-\sum_{k \neq j} s_k u_{i,k}) s_j^{-1} \mod p \text{ for all } 1 \le i \le r) = \\
    \frac{1}{p^r}.
\end{align*}
\end{proof}

\begin{lem} \label{newlem}
Fix $n \ge 2$, and let $u_1, \dots, u_r$ each be chosen uniformly at random from $\left\{0, \dots, p-1\right\}^n$. For any $s, t \in \left\{1, \dots, p-1\right\}^n$ \emph{that are not multiples of each other mod $p$} and for any $z_1, \dots, z_r \in \left\{0, \dots, p-1\right\}$, we have \[\Pr(t \cdot u_i = z_i \mod p \text{ for all } 1 \le i \le r \text{     } | \text{     } s \cdot u_i = z_i \mod p \text{ for all } 1 \le i \le r) = \frac{1}{p^r}.\]
\end{lem}

\begin{proof}
By Lemma \ref{long1g} and the definition of conditional probability, we have

\begin{align*} 
\Pr(t \cdot u_i = z_i \mod p \text{ for all } 1 \le i \le r \text{     } | \text{     } s \cdot u_i = z_i \mod p \text{ for all } 1 \le i \le r) = \\
\frac{\Pr(t \cdot u_i = z_i \mod p \text{ for all } 1 \le i \le r \text{     } \wedge \text{     } s \cdot u_i = z_i \mod p \text{ for all } 1 \le i \le r)}{\Pr(s \cdot u_i = z_i \mod p \text{ for all } 1 \le i \le r)} = \\
p^r \Pr(t \cdot u_i = z_i \mod p \text{ for all } 1 \le i \le r \text{     } \wedge \text{     } s \cdot u_i = z_i \mod p \text{ for all } 1 \le i \le r). 
\end{align*}

Moreover 

\small
\begin{align*}\Pr(t \cdot u_i = z_i \mod p \text{ for all } 1 \le i \le r \text{     } \wedge \text{     } s \cdot u_i = z_i \mod p \text{ for all } 1 \le i \le r) = \\
\frac{|\left\{(u_1, \dots, u_r): \text{     } t \cdot u_i = z_i \mod p \text{     } \wedge \text{     } s \cdot u_i = z_i \mod p \text{     } \wedge \text{     } u_i \in \left\{0, \dots, p-1 \right\}^n \text{ for all } 1 \le i \le r \right\}|}{p^{nr}}.
\end{align*} \normalsize

In order to calculate \small \[|\left\{(u_1, \dots, u_r): \text{     } t \cdot u_i = z_i \mod p \text{     } \wedge \text{     } s \cdot u_i = z_i \mod p \text{     } \wedge \text{     } u_i \in \left\{0, \dots, p-1 \right\}^n \text{ for all } 1 \le i \le r \right\}|,\] \normalsize we must find the number of solutions $(u_1, \dots, u_r)$ to the system of equations $t \cdot u_i = z_i \mod p$ and $s \cdot u_i = z_i \mod p$ for all $1 \le i \le r$.

We form an augmented matrix $M$ with $2r$ rows and $r n + 1$ columns from this system of equations. From left to right, the entries of row $i$ are $(i-1)n$ zeroes, then $s$, then $(r-i)n$ zeroes, then $z_i$ for each $1 \le i \le r$. The entries of row $r+i$ are $(i-1)n$ zeroes, then $t$, then $(r-i)n$ zeroes, then $z_i$ for each $1 \le i \le r$.   

We row-reduce $M$. Since $s$ and $t$ are not multiples of each other mod $p$, $M$ has $2r$ pivot entries. Therefore the system of equations has $2r$ dependent variables and $(n-2)r$ independent variables. There are $p$ choices for each of the independent variables, and the dependent variables are determined by the values of the independent variables, so there are $p^{(n-2)r}$ solutions to the system of equations. Thus 
 \[\Pr(t \cdot u_i = z_i \mod p \text{ for all } 1 \le i \le r \text{     } | \text{     } s \cdot u_i = z_i \mod p \text{ for all } 1 \le i \le r) = \frac{p^r}{p^{2r}} = \frac{1}{p^r}.\]
\end{proof}

\begin{lem}\label{lemma:cartrmodp}
For any subset $S\subset\{1,\dots,p-1\}^n$, there exists $u = (u_1, \dots, u_r)$ with $u_1, \dots, u_r \in \{0, \dots, p-1\}^n$ such that for all $z = (z_1, \dots, z_r) \in \{0, \dots, p-1\}^r$, \[\left\lvert\{x \in S : x \cdot u_i\equiv z_i \pmod p \text{ for all } 1 \le i \le r\}\right\rvert\leq\frac{|S|}{p^r} + 2\sqrt{|S|}.\]
\end{lem}
\begin{proof}
Suppose that $S$ is any subset of $\{1, \dots, p-1\}^n$, and let $u_1, \dots, u_r$ each be chosen uniformly at random from $\{0, \dots, p-1\}^n$. For each $z \in \{0, \dots, p-1\}^r$, let $T_z$ be the set of $x \in S$ for which $x \cdot u_i = z_i$ for all $i$. By linearity of expectation, we have $\EE(|T_z|) = \frac{|S|}{p^r}$ for all $z$.

Consider an arbitrary $z \in \{0, \dots, p-1\}^r$. For each $s \in S$, define the indicator random variable $X_{s,z}$ such that $X_{s,z}=1$ if $s \cdot u_i = z_i$ for all $1 \le i \le r$, and $X_{s,z} = 0$ otherwise. If $s,t \in S$ are not multiples of each other mod $p$, then $\mathrm{Cov}(X_{s,z}, X_{t,z}) = 0$ by Lemmas \ref{long1g} and \ref{newlem}. If $s$ and $t$ are multiples of each other with $s \neq t$, then $\mathrm{Cov}(X_{s,z},X_{t,z}) = \EE(X_{s,z}X_{t,z}) - \EE(X_{s,z})\EE(X_{t,z})$. 

If $z$ contains any nonzero $z_i$, then $\EE(X_{s,z}X_{t,z}) = 0$, giving $\mathrm{Cov}(X_{s,z},X_{t,z}) = -\frac{1}{p^{2r}}$. Thus,
\begin{align*}
    \mathrm{Var}(|T_z|) &= \mathrm{Var}\left(\sum_{s \in S} X_{s,z}\right) = \sum_{s \in S} \mathrm{Var}(X_{s,z}) + \sum_{s \neq t} \mathrm{Cov}(X_{s,z}, X_{t,z})\\
    &\leq \sum_{s \in S} \mathrm{Var}(X_{s,z}) = |S|\left(\frac{1}{p^r} - \frac{1}{p^{2r}}\right) < \frac{|S|}{p^r}.
\end{align*}
By Chebyshev's inequality, $\PP\left(|T_z| \geq \frac{|S|}{p^r} + 2\sqrt{|S|}\right) \leq \frac{1}{4p^r}$.

Otherwise, $\EE(X_{s,z}X_{t,z}) = \frac{1}{p^r}$, giving $\mathrm{Cov}(X_{s,z},X_{t,z}) = \frac{1}{p^r} - \frac{1}{p^{2r}} < \frac{1}{p^r}$. Note that there are at most $(p-2)|S|$ ordered pairs $(s,t)$ for which $s$ and $t$ are multiples of each other $\pmod p$ with $s \neq t$. Thus,
\begin{align*}
    \mathrm{Var}(|T_z|) &= \mathrm{Var}\left(\sum_{s \in S} X_{s,z}\right) = \sum_{s \in S} \mathrm{Var}(X_{s,z}) + \sum_{s \neq t} \mathrm{Cov}(X_{s,z}, X_{t,z})\\
    &\leq \sum_{s \in S} \mathrm{Var}(X_{s,z}) + \frac{(p-2)|S|}{p^r} < \frac{|S|}{p^r} + \frac{(p-2)|S|}{p^r} < \frac{|S|}{p^{r-1}}.
\end{align*}
By Chebyshev's inequality, $\PP\left(|T_z| \geq \frac{|S|}{p^r} + 2\sqrt{|S|}\right) \leq \frac{1}{4p^{r-1}}$.

By the union bound, \[\PP\left(\forall z : |T_z| \leq \frac{|S|}{p^r} + 2\sqrt{|S|}\right) \geq 1 - \frac{p^r-1}{4p^r} - \frac{1}{4p^{r-1}} = \frac{3p^r-p+1}{4p^r} > \frac{1}{2}.\] Thus, the conditions are satisfied with probability greater than $\frac{1}{2}$ when $u$ is chosen randomly, so there must always exist $u$ satisfying the conditions.
\end{proof}

Next we prove the lower bound on the maximum possible multiplicative gap between $\opt_{\operatorname{weak}}(\cart_r(F))$ and $\opt_{\operatorname{std}}(F)$.

\begin{thm}\label{theorem:cartrstdlower}
For all $M > 2r$ and infinitely many $k$, there exists a set $F$ of functions from a set $X$ to a set $Y$ with $|Y|=k$ such that $\opt_{\operatorname{std}}(F) = M$ and \[\opt_{\operatorname{weak}}(\cart_r(F)) \geq (1-o(1))\left(|Y|^r \ln{|Y|}\right)(\opt_{\operatorname{std}}(F)-2r).\]
\end{thm}
\begin{proof}
Fix $n \geq 3$ and $p \geq 5$. For all $a \in \{0,\dots,p-1\}^n$, we define $f_a : \{0,\dots,p-1\}^n \rightarrow \{0,\dots,p-1\}$ so that $f_a(x) = a \cdot x \pmod p$ and define $F_L(p,n) = \{f_a:a \in \{0,\dots,p-1\}^n\}$. It is known that $\opt_{\operatorname{std}}(F_L(p,n)) = n$ for all primes $p$ and $n > 0$ \cite{s-lmlpe-88,Auer1995,Blum1998,LONG2020159}.

We now determine a bound on $\opt_{\operatorname{weak}}(\cart_r(F_L(p,n)))$. Let $S = \{1,\dots,p-1\}^n$, so $|S| = (p-1)^n$. Let $R_1 = \{f_a : a \in S\} \subset F_L(p,n)$. In each round $t > 1$, the adversary will create a list $R_t$ of members of $\{f_a:a \in S\}$ that are consistent with its previous answers. They will always answer \textit{no} and choose $(x_1,\dots,x_r)$ that minimizes \[\max_{(y_1,\dots,y_r)}\left|R_t \cap \{f: f(x_i) = y_i \text{ for all } 1 \le i \le r\}\right|.\]

By Lemma~\ref{lemma:cartrmodp}, we have \[R_{t+1} \geq |R_t| - \frac{|R_T|}{p^r} - 2\sqrt{|R_t|} \geq |R_t| - \frac{|R_t|}{p^r} - \frac{2|R_t|}{p^r\sqrt{\ln{p}}} = \left(1-\frac{1+\frac{2}{\sqrt{\ln{p}}}}{p^r}\right) |R_t|\] as long as $|R_t| \geq p^{2r}\ln{p}$. Thus, we have $|R_t| \geq \left(1-\frac{1+\frac{2}{\sqrt{\ln{p}}}}{p^r}\right)^{t-1}(p-1)^n$. Therefore, whenever $\left(1-\frac{1+\frac{2}{\sqrt{\ln{p}}}}{p^r}\right)^{t-1}(p-1)^n \geq p^{2r} \ln{p}$, the adversary can guarantee $t$ wrong guesses. This is true for $t = (1-o(1))n p^r\ln{p}$, which gives the desired result.
\end{proof}

\begin{remark}
    Since we have the trivial inequality $\opt_{\operatorname{amb},r}(F)\geq\opt_{\operatorname{weak}}(\cart_r(F))$ because the learner has strictly more information in the scenario on the right hand side, we also have $\opt_{\operatorname{amb},r}(F) \geq (1-o(1))\left(|Y|^r \ln{|Y|}\right)(\opt_{\operatorname{std}}(F)-2r)$ for the families $F = F_L(p,n)$.
\end{remark}

Next we establish a similar upper bound relating $\opt_{\operatorname{weak}}(\cart_r(F))$ and $\opt_{\operatorname{std}}(F)$. For this bound, we use the fact that for all sets $F$ of functions $f:X\to Y$, we have $\opt_{\operatorname{std}}(\cart_r(F))=\opt_{\operatorname{std}}(F)$. We also use the bound $\opt_{\operatorname{weak}}(F)\leq(1+o(1))(|Y|\ln|Y|)\opt_{\operatorname{std}}(F)$ which was proved in \cite{LONG2020159}.

\begin{thm}\label{theorem:cartrstdupper}
For any set $F$ of functions from some set $X$ to $\{0,1,\dots,k-1\}$ and for any $r \geq 1$, \[\opt_{\operatorname{weak}}(\cart_r(F)) \leq (1+o(1))\left(|Y|^r r\ln{|Y|}\right)\opt_{\operatorname{std}}(F).\]
\end{thm}

\begin{proof}
Substituting $\cart_r(F)$ for $F$ (and therefore setting $Y^r$ in place of $Y$) in the upper bound from \cite{LONG2020159} and using the fact that $\opt_{\operatorname{std}}(\cart_r(F))=\opt_{\operatorname{std}}(F)$, we get
\begin{align*}
\opt_{\operatorname{weak}}(\cart_r(F))&\leq(1+o(1))(|Y|^r\ln\left(|Y|^r\right))\opt_{\operatorname{std}}(\cart_r(F))\\
&=(1+o(1))(|Y|^r r\ln|Y|)\opt_{\operatorname{std}}(F).
\end{align*}
\end{proof}

\begin{remark}
    Theorem~\ref{theorem:cartrstdlower} demonstrates that the upper bound in Theorem~\ref{theorem:cartrstdupper} is sharp up to a factor of $r(1+o(1))$.
\end{remark}

Next we prove an upper bound for $\opt_{\operatorname{amb,r}}(F)$ using a method from \cite{LONG2020159}. Like the last bound, this one is sharp up to a factor of $r(1+o(1))$.

\begin{thm}\label{theorem:ambrlongupper}
For any set $F$ of functions from some set $X$ to $\{0,1,\dots,k-1\}$ and for any $r \geq 1$, \[\opt_{\operatorname{amb,r}}(F) \leq (1+o(1))\left(|Y|^r r\ln{|Y|}\right)\opt_{\operatorname{std}}(F).\]
\end{thm}

\begin{proof}
Consider an algorithm $B$ for the $r$-input delayed, ambiguous reinforcement model which uses a worst-case optimal algorithm $A_s$ for the standard model and maintains copies of $A_s$ which are given different inputs and answers. In each round, $B$ chooses a prediction by taking a weighted vote over the predictions of the copies.

Fix $\alpha = \frac{1}{k^r \ln{k}}$. Each copy of $A_s$ gets a weight, where the current weight is $\alpha^x$ if the copy has contributed to $B$ making $x$ mistakes in the earlier weighted votes. $B$ uses these weights to make a prediction for the current round by taking a weighted vote over the predictions of each copy for the outputs of all $r$ inputs. In case of ties, the winner is chosen uniformly at random among the predictions that tied for the highest weight.

At the beginning, $B$ starts with one copy of $A_s$. Whenever it gets a wrong answer for the outputs of the $r$ inputs in the current round, any copy of $A_s$ which predicted an answer that did not win the weighted vote is rewound as if the round did not happen, so they forget the input and their answer. The copies of $A_s$ that predicted the wrong answer which won the weighted vote are cloned to make $k^r-1$ copies, and each copy is given a different answer for the outputs of the $r$ inputs, which differs from the wrong answer which won the weighted vote.

If $W_t$ is the total weight of the copies of $A_s$ before round $t$, we must have $W_t \ge \alpha^{\opt_{\operatorname{std}}(F)}$ since one copy of $A_s$ always gets the correct answers. Moreover, if $B$ makes a mistake in round $t$, then copies of $A_s$ with total weight at least $\frac{W_t}{k^r}$ are cloned to make $k^r-1$ copies which each have weight $\alpha$ times their old weight. This implies that $W_{t+1} \le (1-\frac{1}{k^r})W_t + \frac{1}{k^r}(\alpha(k^r-1)W_t) < (1-\frac{1}{k^r})W_t + \alpha W_t$.

Thus after $B$ has made $x$ mistakes, we have $W_t < (1-\frac{1}{k^r}+\alpha)^x < e^{-(\frac{1}{k^r}-\alpha)x}$. Therefore $e^{-(\frac{1}{k^r}-\alpha)x} > \alpha^{\opt_{\operatorname{std}}(F)}$, so $x < \frac{\ln(\frac{1}{\alpha})\opt_{\operatorname{std}}(F)}{\frac{1}{k^r}-\alpha}$, which implies the desired bound.
\end{proof}

\section{New Models for Permutation Functions}\label{section:permutation}

In this section, we define and explore new models where the family of possible functions $F$ is a set of permutations of length $n$ and where the learner tries to guess information about the relative orders of inputs.

\subsection{The Order Model}

We first define a new variant model called the order model.
\begin{dfn}
In the \emph{order model}, for a set $F$ of permutations of $n$ numbers, the learner tries to guess a permutation function $f\in F.$ On each turn, the adversary chooses $r$ distinct inputs to $f$, i.e., a set $S\subseteq[n]$ with $\abs{S}=r$. The learner guesses the permutation of $\left\{1, \dots, r\right\}$ which is order-isomorphic to the outputs of $f$ on the given inputs.

Under weak reinforcement, the adversary informs the learner if they made a mistake, and under strong reinforcement, the adversary gives the correct answer to the learner. We denote the worst-case amount of mistakes for the learner with weak reinforcement as $\opt_{\operatorname{weak, <, r}}(F)$ and with strong reinforcement as $\opt_{\operatorname{strong, <, r}}(F).$ Note that $\opt_{\operatorname{strong, <, r}}(F)\le\opt_{\operatorname{weak, <, r}}(F).$ If $r=2,$ then strong and weak reinforcement are identical, so equality holds.
\end{dfn}

We first find an upper bound for the order model by presenting a strategy for the learner which is analogous to Theorem~\ref{theorem:ambrupper}.

\begin{thm}
\label{theorem:orderupper}
For $r>1,$ $\opt_{\operatorname{weak, <, r}}(F)<r!\ln\abs{F}.$
\end{thm}
\begin{proof}
For each input that the adversary gives, the learner can pick the answer that corresponds to the most possible permutations. After each incorrect guess, at least an $\frac1{r!}$ fraction of the previously possible permutations get eliminated. Therefore, the number of incorrect guesses, and consequently the number of mistakes, is at most $\ln_{\frac{r!}{r!-1}}\abs{F}<r!\ln\abs{F}$.
\end{proof}

When $F = S_n$, Theorem \ref{theorem:orderupper} shows that $\opt_{\operatorname{weak, <, 2}}(S_n)<2n \ln{n}.$ We find a lower bound on $\opt_{\operatorname{weak, <}}(S_n)$ which is within a factor of $2+o(1)$ of the upper bound. In order to prove this bound, we define a function called $p(n)$ and prove a lemma about it.

\begin{dfn}
For $n\ge 1,$ we let $v_n:=\floor{\log_2 n}$ and define $p(n):=\displaystyle\sum_{1\le m\le n}v_m$.
\end{dfn}

\begin{lem}
For $n\ge 1,$ $p(n)=(n+1)v_n-2(2^{v_n}-1).$
\end{lem}
\begin{proof}
We prove this by induction. The base case $n=1$ holds.

If $n$ is not a power of two, then $v_{n-1}=v_n,$ so $p(n)=p(n-1)+v_n=(nv_n-2(2^{v_n}-1))+v_n=(n+1)v_n-2(2^{v_n}-1)$, as desired.

Otherwise, $n=2^n$ is a power of two, so $v_{n-1}=v_n-1$ and $p(2^{v_n})=p(2^{v_n}-1)+v_n=(2^{v_n}(v_n-1)-2(2^{v_n-1}-1))+v_n=(2^{v_n}+1)v_n-2(2^{v_n}-1)$, as desired.
\end{proof}

Now we present a strategy resembling insertion sort for the adversary which achieves a lower bound of $p(n).$

\begin{thm} \label{theorem:perm2} Under the order model, the adversary can achieve $\opt_{\operatorname{weak, <, 2}}(S_n) \ge p(n)$ for all $n \ge 2$. \end{thm}

\begin{proof}
The adversary withholds any inquiries about $f(i)$ until the order of the smaller inputs $j<i$ is known. We use induction to show the desired bound. The base case of $n=2$ clearly holds, since $p(2) = 1$.

For $n>2,$ by the inductive hypothesis, the adversary can force the learner to make at least $p(n-1)$ mistakes without learning anything about $f(n).$ Assume without loss of generality that $f(1),\dots,f(n-1)$ are in increasing order.

The adversary then prolongs the learner from finding the position of $f(n)$ by making the learner do a binary search. Specifically, if at some point the learner's bounds on $f(n)$ are $a<f(n)\le b,$ the adversary asks about $(n,m)$ where $m$ is $\frac{a+b}2$ rounded to the nearest integer, and says \textit{no} to the learner's prediction. In this way, the adversary ensures that the number of possible values for $f(n)$ is at least $\floor{\frac{b-a}2}$ after the learner's guess. Since the number of possible values starts at $n,$ the adversary will be able to guarantee $v_n$ mistakes to find the value of $f(n).$ Thus, $\opt_{\operatorname{weak, <, 2}}(S_n)\ge \opt_{\operatorname{weak, <, 2}}(S_{n-1})+v_n\ge p(n-1)+v_n=p(n)$, as desired.
\end{proof}

We also get a lower bound on $\opt_{\operatorname{weak, <, r}}(S_n)$ for $r>2$ with a strategy that resembles merge sort. This lower bound is within a $(1+o(1))r \ln{r}$ factor of the upper bound $\opt_{\operatorname{weak, <, r}}(S_n)<(1-o(1))r! n \ln{n}$ which follows from Theorem \ref{theorem:orderupper}.

\begin{thm}\label{theorem:merge}
The adversary can achieve $\opt_{\operatorname{weak, <, r}}(S_n) \ge (1-o(1))(r-1)! n\log_r n$ as $n \to \infty$ and then $r \to \infty$.
\end{thm}

\begin{proof}
We use strong induction on $n.$

First, the adversary divides $[n]$ into $\floor{\frac nr}$ sets $S_i$ each of size $r$ (ignoring any remaining elements). For each $i,$ the adversary repeatedly asks for ordering of $S_i,$ saying ``NO'' each time until the order is known by the learner. This takes a total of $\floor{\frac nr}(r!-1)$ mistakes.

Then, the adversary forms sets $C_j$ from the relative $j^{th}$ elements of each set, and uses the induction hypothesis on $n$; note that knowing any of $C_j$’s orders does not eliminate possibilities for the orders of other $C_i$. This gives a recursion of the form $\opt_{\operatorname{weak, <, r}}(S_n)\ge (1-o(1))n(r-1)!+r \opt_{\operatorname{weak, <, r}}(S_{\floor{\frac nr}}) =(1-o(1))(r-1)!n\log_r(n)$, as desired.
\end{proof}

\subsection{The Comparison Model}\label{section:comp}

We define a second variant called the comparison model.
\begin{dfn}
In the \emph{comparison model}, for a set $F$ of permutations of $n$ numbers, the learner tries to guess a permutation function $f\in F.$ On each turn, the adversary chooses $r$ distinct pairs of inputs to $f$, i.e., a set $S\subseteq\{(i,j):i,j\in [n],i\ne j\}$ with $\abs{S}=r$. The learner guesses for each pair $(i,j)$ whether or not $f(i)<f(j)$.

Under weak reinforcement, the adversary informs the learner if they made a mistake, and under strong reinforcement, the adversary gives the correct choices to the learner. We denote the worst-case amount of mistakes for the learner with weak reinforcement as $\opt_{\operatorname{weak, c, r}}(F)$ and with strong reinforcement as $\opt_{\operatorname{strong, c, r}}(F).$ Note that $\opt_{\operatorname{strong, c, r}}(F)\le\opt_{\operatorname{weak, c, r}}(F).$ If $r=1,$ then both sides of the equation are equal to $\opt_{\operatorname{strong, <, 2}}(F).$
\end{dfn}

We similarly get the following upper bound.

\begin{thm}
If $F\subseteq S_n$, $\opt_{\operatorname{weak, c, r}}(F)<2^r\cdot\ln |F|.$
\end{thm}

\begin{proof}
For each input that the adversary gives, the learner can pick the answer that corresponds to the most possible permutations. After each incorrect guess, at least a $\frac1{2^r}$ fraction of the previously possible permutations get eliminated. Therefore, the number of incorrect guesses, and consequently the number of mistakes, is at most $\log_{\frac{2^r}{2^r-1}}(|F|)<2^r\ln |F|.$
\end{proof}

We present a strategy imitating quicksort to obtain the following lower bound.

\begin{thm}
The adversary can achieve $\opt_{\operatorname{weak, c, r}}(S_n) \ge (1-o(1))\frac{2^r}{r} n\log_2 n$ as $n \to \infty$ and then $r \to \infty$.
\end{thm}

\begin{proof}
The adversary first splits $[n-1]$ into $\floor{\frac{n-1}r}$ sets of size $r$ with some leftover elements. For every such set $\abs{S}$ with $\abs{S}=r$, the adversary then queries the set of pairs $\qty{(s,n):s\in S}$ at least $2^r$ times, responding \textit{no} each time until the answer is forced. After all queries, the learner has made at least $\floor{\frac{n-1}r}(2^r-1)=(1-o(1))\frac{2^r}{r} n$ mistakes.

The learner now at most knows the set of inputs $i$ for which $\pi(i)<\pi(n)$. The adversary can therefore force the learner to make at least $\opt_{\operatorname{weak, c, r}}(S_{\pi(n)-1})+\opt_{\operatorname{weak, c, r}}(S_{n-\pi(n)})$ more mistakes.

The total number of mistakes is thus at least $2\qty((1-o(1)) \frac{2^r}r \cdot \frac n2  (\log_2 n-1))+(1-o(1))\frac{2^r}r n=(1-o(1))\frac{2^r}r n\log_2 n$ in total, as desired.
\end{proof}

\subsection{The Selection Model}\label{section:sel}

We define another variant called the selection model.

\begin{dfn}
In the \emph{selection model}, for a set $F$ of permutations of $n$ numbers, the learner tries to guess a permutation function $f\in F.$ On each turn, the adversary chooses $r$ (not necessarily distinct) inputs to $f$, and the learner guesses the input $x$ among them that maximizes $f(x)$.

Under weak reinforcement, the adversary informs the learner if they made a mistake, and under strong reinforcement, the adversary gives the correct choices to the learner. We denote the worst-case amount of mistakes for the learner with weak reinforcement as $\opt_{\operatorname{weak, s, r}}(F)$ and with strong reinforcement as $\opt_{\operatorname{strong, s, r}}(F).$ Note that $\opt_{\operatorname{strong, s, r}}(F)\le\opt_{\operatorname{weak, s, r}}(F).$ If $r=2,$ then both sides of the equation are equal to $\opt_{\operatorname{strong, <, 2}}(F).$
\end{dfn}

Note that because the inputs the adversary chooses do not have to be distinct, the adversary can effectively ask about any set of inputs of size at most $r$ by using duplicates.

\begin{thm}\label{theorem:selectionupper}
For $r>1$, \[\opt_{\operatorname{weak, s, r}}(F)<r\ln|F|.\]
\end{thm}

\begin{proof}
For each input that the adversary gives, the learner can pick the input that is maximal for the most possible remaining permutations. After each incorrect guess, at least a $\frac1r$ fraction of the previously possible permutations get eliminated. Therefore, the number of incorrect guesses, and consequently the number of mistakes, is at most $\log_{\frac r{r-1}}(|F|)<r\ln |F|.$
\end{proof}

We again imitate merge sort to provide a strategy for the adversary.

\begin{thm}\label{theorem:selectionlower}
The adversary can achieve $\opt_{\operatorname{weak, s, r}}(S_n)\ge (1-o(1))nr\log_{r}(n)/2$ as $n \to \infty$ and then $r \to \infty$.
\end{thm}

\begin{proof}
We use strong induction on $n$.

First, the adversary divides $[n]$ into $r$ sets $S_i$ each of size $\floor{\frac nr}$ (ignoring any remaining elements). For each $i,$ the adversary uses the optimal strategy for $\floor{\frac nr}$ on $S_i$, forcing a total of at least $r\opt_{\operatorname{weak, s, r}}(S_{\floor{\frac nr}})$ mistakes.

Then, the adversary repeatedly asks about the elements $x_i$ of each corresponding $S_i$ for $1\le i\le r$ for which $f(x_i)$ is maximal, saying ``NO'' each time until their maximal element, say $x_m$, is known. Then, the adversary deletes $x_m$ from $S_m$ and repeats this process until all sets are empty, which takes $r\floor{\frac nr}$ iterations in total.

The number of tries it takes the learner each time to guess correctly is the number of possibilities for the maximal remaining element, which is the number of nonempty $S_i$'s remaining. Therefore, until one of the $S_i$'s is empty, which takes at least $|S_i|=\floor{\frac nr}$ turns, the adversary can force the learner to make at least $(r-1)$ mistakes, and after that $(r-2)$, and so on all the way to $0$ mistakes when all but one $S_i$ is empty. Therefore, the learner throughout the process makes an average of at least $(r-1)/2$ mistakes per iteration, for $r\floor{\frac nr}\cdot (r-1)/2=(1-o(1))nr/2$ mistakes total.

Thus, $\opt_{\operatorname{weak, s, r}}(S_n)\ge (1-o(1))nr/2+r\opt_{\operatorname{weak, s, r}}(S_{\floor{\frac nr}})=(1-o(1))nr\log_{r}(n)/2$, as desired.
\end{proof}

\subsection{The Relative Position Model}\label{section:rpos}

Next we define a variant called the relative position model.
\begin{dfn}
In the \emph{relative position model}, for a set $F$ of permutations of $n$ numbers, the learner tries to guess a permutation function $f\in F.$ On each turn, the adversary chooses a set $S$ of $r$ distinct inputs to $f$ and an element $x\notin S,$ and asks about the pair $(x,S).$ The learner guesses the relative position of $f(x)$ in the permutation of $\left\{1, \dots, r+1\right\}$ which is order-isomorphic to the outputs of $f$ on $\qty{x}\cup S.$

Under weak reinforcement, the adversary informs the learner if they made a mistake, and under strong reinforcement, the adversary gives the correct position to the learner. We denote the worst-case amount of mistakes for the learner with weak reinforcement as $\opt_{\operatorname{weak, p, r}}(F)$ and with strong reinforcement as $\opt_{\operatorname{strong, p, r}}(F).$ Note that $\opt_{\operatorname{strong, p, r}}(F)\le\opt_{\operatorname{weak, p, r}}(F).$ If $r=1,$ then both sides of the equation are equal to $\opt_{\operatorname{strong, <, 2}}(F).$
\end{dfn}

We again imitate Theorem~\ref{theorem:ambrupper} to obtain an upper bound for the relative position model.

\begin{thm}
\label{theorem:relupper}
If $F\subseteq S_n$, $\opt_{\operatorname{weak, p, r}}(F)<(r+1)\ln |F|.$
\end{thm}

\begin{proof}
For each input that the adversary gives, the learner can pick the answer that corresponds to the most possible permutations. After each incorrect guess, at least a $\frac1{r+1}$ fraction of the previously possible permutations get eliminated. Therefore, the number of incorrect guesses, and consequently the number of mistakes, is at most $\log_{\frac{r+1}r}(|F|)<(r+1)\ln |F|.$
\end{proof}

As with Theorem \ref{theorem:perm2}, we use a strategy resembling insertion sort to obtain the following result.

\begin{thm}\label{theorem:namedname}
 Under the relative position model, $\lim_{r \to \infty} 
    \lim_{n \to \infty} \frac{\opt_{\operatorname{weak, p, r}}(S_n)}{rn\ln n} = 1$.
\end{thm}

\begin{proof}
The upper bound follows from Theorem~\ref{theorem:relupper}. For the lower bound, we use strong induction on $n.$ The adversary withholds any inquiries about $f(i)$ until the order of the smaller inputs $j<i$ is known. Assume without loss of generality that $f(1),\dots,f(n)$ are in increasing order. Let $g_i$ be the function on $[n]$ that maps $m$ to $1$ if $m\ge i$ and $0$ otherwise, and let $G$ be the family of these functions. Note that these functions are non-decreasing.

The remaining problem for the learner is equivalent to guessing $g_{f(n+1)},$ where the adversary queries $r$ values at a time. Thus by Theorem \ref{theorem:cartrlower}, the adversary can force the learner to make at least $\opt_{\operatorname{weak}}(\cart_r(G))=(1-o(1))r\ln n$ mistakes on $g_{f(n+1)},$ as desired.
\end{proof}

In a similar fashion, we can prove a similar result for pattern-avoiding permutations.

\begin{thm} \label{theorem:avoidance}
If $S_{n, \pi}$ is the set of $\pi$-avoiding permutations of length $n$, then $\opt_{\operatorname{weak, p, r}}(S_{n, \pi})\ge(1-o(1))rn\ln(k-1)$ as $n \to \infty$ and then $r \to \infty$, where $k = |\pi|$ denotes the length of $\pi$.
\end{thm}
\begin{proof}
The adversary withholds any inquiries about $f(i)$ until the order of the smaller inputs $j<i$ is known. We use induction to show the desired bound. 

Let $N$ be the set of possible values of $f(n+1)$. Any number less than $\pi(k)$ or more than $n-k+\pi(k)$ must be in $N$ as it would not be able to form the permutation pattern $\pi$, so $|N|\ge \pi(k) +k -1 -\pi(k) = k-1$. Let the elements of $N$ in increasing order be $s_1,\dots,s_{|N|}$.

Let $g_i$ be the function on $[|N|]$ that maps $m$ to $1$ if $m\ge s_i$ and $0$ otherwise, and let $G$ be the family of these functions. Note that these functions are non-decreasing.

The remaining problem for the learner is equivalent to guessing $g_{f(n+1)},$ where the adversary queries $r$ values at a time. Thus by Theorem \ref{theorem:cartrlower}, the adversary can force the learner to make at least $\opt_{\operatorname{weak}}(\cart_r(G))=(1-o(1))r\ln |N|\ge(1-o(1))r\ln (k-1)$ mistakes, as desired.
\end{proof}

When $\pi=I_k$, the identity permutation, the size of $S_{n, \pi}$ is $(k-1)^{O(n)}$ \cite{REGEV1981115}. The next result follows from combining Theorems \ref{theorem:avoidance} and \ref{theorem:relupper}.

\begin{cor}\label{cor:identity}
For $S_{n, I_k}$ the family of $I_k$-avoiding permutations of length $n$, $\opt_{\operatorname{weak, p, r}}(S_{n, I_k})=\Theta(rn\ln k)$ as $n \to \infty$ and then $r \to \infty$.
\end{cor}

\subsection{The Delayed Relative Position Model}

In this subsection, we define delayed reinforcement for the relative position model.

\begin{dfn}
In the \emph{delayed relative position model}, for a set $F$ of permutations of $n$ numbers, the learner tries to guess a permutation function $f\in F.$ On each turn, the adversary picks an input $x$ and proceeds to give the $r$ elements of a set $S$ one by one (with the requirement that $x\notin S$). After each of the adversary's inquiries, the learner guesses either \textit{higher} or \textit{lower}. At the end of each round, the learner's final guess for the relative position of $x$ in $S$ is one plus the number of times they said \textit{higher}.

Under weak reinforcement, the adversary informs the learner if their final guess is incorrect, and under strong reinforcement, the adversary gives the correct position to the learner. We denote the worst-case amount of mistakes for the learner with weak reinforcement as $\opt_{\operatorname{wrpos},r}(F)$ and with strong reinforcement as $\opt_{\operatorname{srpos},r}(F).$ Note that $\opt_{\operatorname{srpos},r}(F)\le\opt_{\operatorname{wrpos},r}(F),$ and that if $r=1,$ both sides of the equation are equal to $\opt_{\operatorname{strong, <, 2}}(F).$
\end{dfn}

We state the analogs of the relative position model results for the delayed version.
\begin{thm}
Given a set of permutations $F\subseteq S_n$ and a positive integer $r$, these analogs of the results in Section \ref{section:rpos} hold:
\begin{itemize}
    \item Analog of Theorem \ref{theorem:relupper}: $\opt_{\operatorname{wrpos},r}(F)<2^r\ln |F|$.
    \item Analog of Theorem \ref{theorem:namedname}: $\lim_{r \to \infty} 
    \lim_{n \to \infty} \frac{\opt_{\operatorname{wrpos},r}(S_n)}{2^r n\ln n} = 1$.
    \item Analog of Theorem \ref{theorem:avoidance}: $\opt_{\operatorname{wrpos},r}(S_{n, \pi})\ge(1-o(1))2^r n\ln(k-1)$ as $n \to \infty$ and then $r \to \infty$, where $k = |\pi|$ denotes the length of $\pi$.
    \item Analog of Corollary \ref{cor:identity}: $\opt_{\operatorname{wrpos},r}(S_{n, I_k})=\Theta(2^rn\ln k)$ as $n \to \infty$ and then $r \to \infty$.
\end{itemize}
\end{thm}

The proofs of these results are nearly identical to those of the previous section, with $r+1$ replaced by $2^r$ and cited results from Section \ref{section:boundcartr} replaced with analogous results from Section \ref{section:boundambr}, so  we omit them here. 

\section{Future Work}\label{section:future}
In this final section, we outline some possible areas for future work based on the results in our paper.

In Sections~\ref{section:boundcartr} and \ref{section:boundambr}, we mainly focused on the price of feedback for non-decreasing $F$ in the $r$-input weak reinforcement model and the $r$-input delayed, ambiguous reinforcement model. What bounds can be obtained for general sets of functions $F$? Moreover, is it possible to obtain a sharper bound on the maximum possible value of $\frac{\opt_{\operatorname{amb},r}(F)}{\opt_{\operatorname{weak}}(\cart_r(F))}$ over all families of functions $F$ with $|F| > 1$? 

In a preliminary version of this paper (\cite{fls}), we found rough bounds for the $r$-input weak reinforcement model and the $r$-input delayed, ambiguous reinforcement model on non-decreasing multi-class functions. In particular we proved that $\opt_{\operatorname{weak}}(\cart_r(F)) \leq {{r+k-1} \choose k-1} \ln(|F|)$ and $\opt_{\operatorname{weak}}(\cart_r(F))\geq \frac{1}{2^{k-2}}(1-o(1))r\ln(|F|)$ for any subset of functions $F$ of the non-decreasing functions from $X$ to $\{0, 1, \dots, k-1\}$. We also showed in \cite{fls} that $\opt_{\operatorname{amb},r}(F)<\min(k^r\ln(|F|),|F|)$ and $\opt_{\operatorname{amb},r}(F)\geq (1-o(1))2^{r-k+2}\ln(|F|)$ for all $X$ and $F$. Since the preceding upper and lower bounds for multi-class functions have large gaps, a natural problem is to narrow these gaps.

In Theorem \ref{theorem:merge}, we described an adversary strategy for the order model that gives a bound of $\opt_{\operatorname{weak, <, r}}(S_n) \ge(1-o(1))(r-1)! n\log_r n$, whereas Theorem \ref{theorem:orderupper} gives an upper bound of $\opt_{\operatorname{weak, <, r}}(S_n)<r!\ln n!=(1-o(1))r! n\ln n$. We conjecture that the latter bound is sharp, i.e., that $\opt_{\operatorname{weak, <, r}}(S_n)=(1-o(1))r! n\ln n$.


\section*{Acknowledgments}

We would like to thank the MIT PRIMES-USA program and everyone involved for providing us with this research opportunity.

\newpage

\printbibliography[heading=bibintoc]
\end{document}